\documentclass{article}

\usepackage{microtype}
\usepackage{graphicx}
\usepackage{subfigure}
\usepackage{booktabs} 
\usepackage{color}

\usepackage{amssymb}
\usepackage{amsfonts}
\usepackage{amsthm}
\usepackage{amsmath} 
\newtheorem{theorem}{Theorem}
\newtheorem{lemma}{Lemma}
\newtheorem{corollary}{Corollary}
\newtheorem{definition}{Definition}

\usepackage{bbm}
\usepackage{tikz}
\usetikzlibrary{decorations.pathreplacing}
\usetikzlibrary{shapes.arrows}
\usetikzlibrary{patterns}
\usetikzlibrary{calc}
\usetikzlibrary{shapes.geometric,calc,positioning,decorations}

\usepackage{multirow}
\usepackage{multicol}
\usepackage{verbatim}

\makeatletter
\tikzset{
    database/.style={
        path picture={
            \draw (0, 1.5*\database@segmentheight) circle [x radius=\database@radius,y radius=\database@aspectratio*\database@radius];
            \draw (-\database@radius, 0.5*\database@segmentheight) arc [start angle=180,end angle=360,x radius=\database@radius, y radius=\database@aspectratio*\database@radius];
            \draw (-\database@radius,-0.5*\database@segmentheight) arc [start angle=180,end angle=360,x radius=\database@radius, y radius=\database@aspectratio*\database@radius];
            \draw (-\database@radius,1.5*\database@segmentheight) -- ++(0,-3*\database@segmentheight) arc [start angle=180,end angle=360,x radius=\database@radius, y radius=\database@aspectratio*\database@radius] -- ++(0,3*\database@segmentheight);
        },
        minimum width=2*\database@radius + \pgflinewidth,
        minimum height=3*\database@segmentheight + 2*\database@aspectratio*\database@radius + \pgflinewidth,
    },
    database segment height/.store in=\database@segmentheight,
    database radius/.store in=\database@radius,
    database aspect ratio/.store in=\database@aspectratio,
    database segment height=0.1cm,
    database radius=0.25cm,
    database aspect ratio=0.35,
}
\makeatother

\makeatletter
\pgfdeclaredecoration{penciline}{initial}{
    \state{initial}[width=+\pgfdecoratedinputsegmentremainingdistance,auto corner on length=1mm,]{
        \pgfpathcurveto%
        {
            \pgfqpoint{\pgfdecoratedinputsegmentremainingdistance}
                            {\pgfdecorationsegmentamplitude}
        }
        {
        \pgfmathrand
        \pgfpointadd{\pgfqpoint{\pgfdecoratedinputsegmentremainingdistance}{0pt}}
                        {\pgfqpoint{-\pgfdecorationsegmentaspect\pgfdecoratedinputsegmentremainingdistance}%
                                        {\pgfmathresult\pgfdecorationsegmentamplitude}
                        }
        }
        {
        \pgfpointadd{\pgfpointdecoratedinputsegmentlast}{\pgfpoint{1pt}{1pt}}
        }
    }
    \state{final}{}
}
\makeatother

\usepackage{xcolor}
\definecolor{tomato}{HTML}{FF0000}
\definecolor{dodgerblue}{HTML}{0000FF}
\definecolor{orange}{HTML}{00FF00}
\definecolor{lime}{HTML}{FFA500}
\usepackage{hyperref}

\usepackage[accepted]{icml2020}

\icmltitlerunning{Secure Data Sharing With Flow Model}

\begin{document}

\twocolumn[
\icmltitle{
Secure Data Sharing With Flow Model
}



\icmlsetsymbol{equal}{*}

\begin{icmlauthorlist}
\icmlauthor{Chenwei Wu}{equal,duke}
\icmlauthor{Chenzhuang Du}{equal,iiis}
\icmlauthor{Yang Yuan}{iiis}
\end{icmlauthorlist}

\icmlaffiliation{duke}{Department of Computer Science, Duke University, Durham, North Carolina, USA}
\icmlaffiliation{iiis}{Institute for Interdisciplinary Information Sciences, Tsinghua University, Beijing, China}

\icmlcorrespondingauthor{Chenwei Wu}{cwwu@cs.duke.edu}
\icmlcorrespondingauthor{Chenzhuang Du}{ducz20@mails.tsinghua.edu.cn}
\icmlcorrespondingauthor{Yang Yuan}{yuanyang@tsinghua.edu.cn}

\icmlkeywords{Machine Learning, ICML}

\vskip 0.3in
]


\printAffiliationsAndNotice{\icmlEqualContribution} 
\begin{abstract}
In the classical multi-party computation setting,
multiple parties jointly compute a function  without revealing their own input data. We consider a variant of this problem, where the input data can be shared for machine learning training purposes,
but the data are also encrypted so that they cannot be recovered by other parties.  
We present a rotation based method using flow model, and theoretically justified its security.
We demonstrate the effectiveness of our method in different scenarios, including supervised secure model training, and unsupervised generative model training. Our code is available at \url{https://github.com/duchenzhuang/flowencrypt}.
\end{abstract}
\section{Introduction}
\label{intro}
Data and model are the two most important factors in machine learning. 
Unfortunately, 
sometimes no one has both of them, so 
multiple parties have to collaborate together to solve certain problems.  
For the data providers, such collaboration can be risky: once the data is sent to the other parties, they may later sell or use the data without getting permissions from the data providers. This risk is becoming larger in recent years due to the larger volume of data owned by the providers.

Multi-party computation (MPC)~\cite{yao1982protocols,yao1986generate,goldreich2019play,chaum1988multiparty,ben2019completeness,bogetoft2009secure} is specifically designed for such scenarios. With rigorous theoretical guarantees, it provides one way to let multiple parties jointly compute any function and also keep each party's data private. However, this general framework comes with a price: the computational overhead of MPC is so high that currently it cannot yet support large scale computational task like neural network training.







\begin{figure*}
    \centering
    \includegraphics[width=0.8\linewidth]{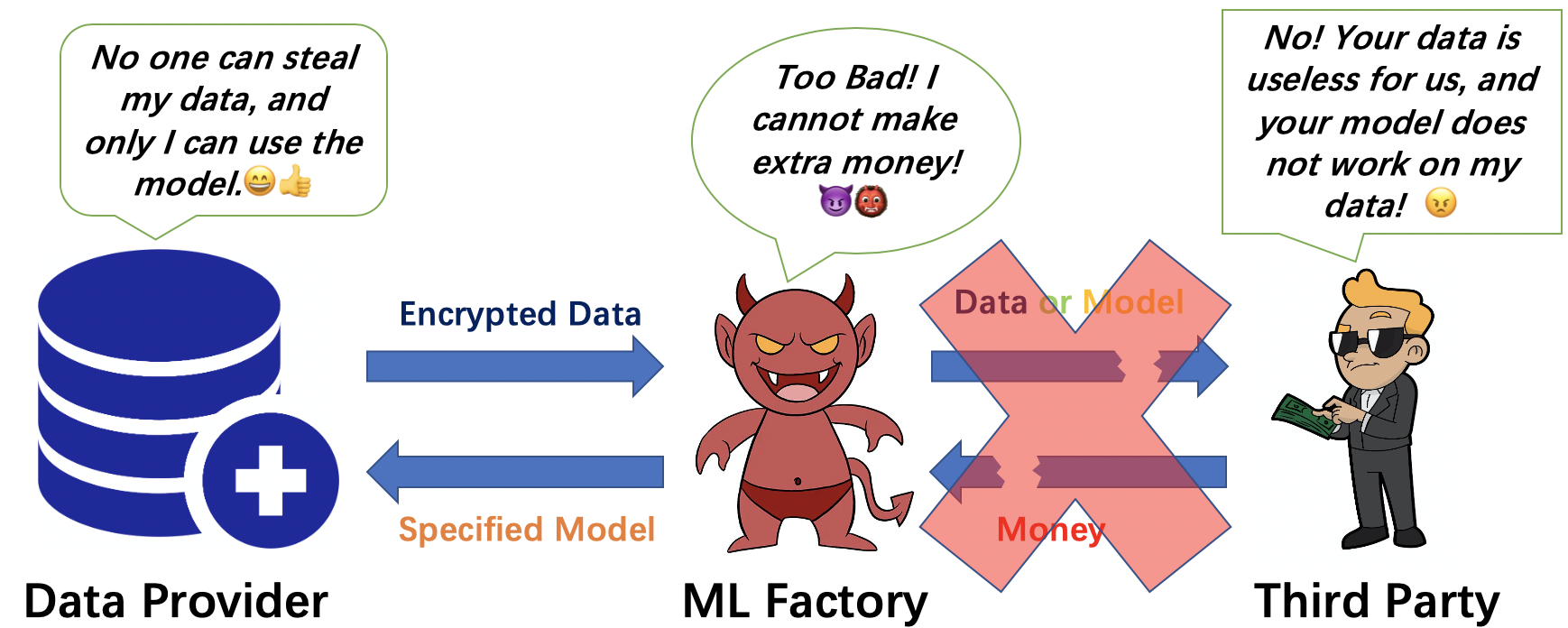}
    \caption{Illustration of the supervised secure model training scenario. The data provider wants to encrypt its private data so that the ML factory can use it to train a model, but the model can only be used by the data provider with the secret encryption key. Moreover, the data cannot be decrypted, so the ML factory cannot sell the data or the model to the third party. }
    \label{fig:setting}
\end{figure*}

\begin{figure}
    \centering
    \includegraphics[width=\linewidth]{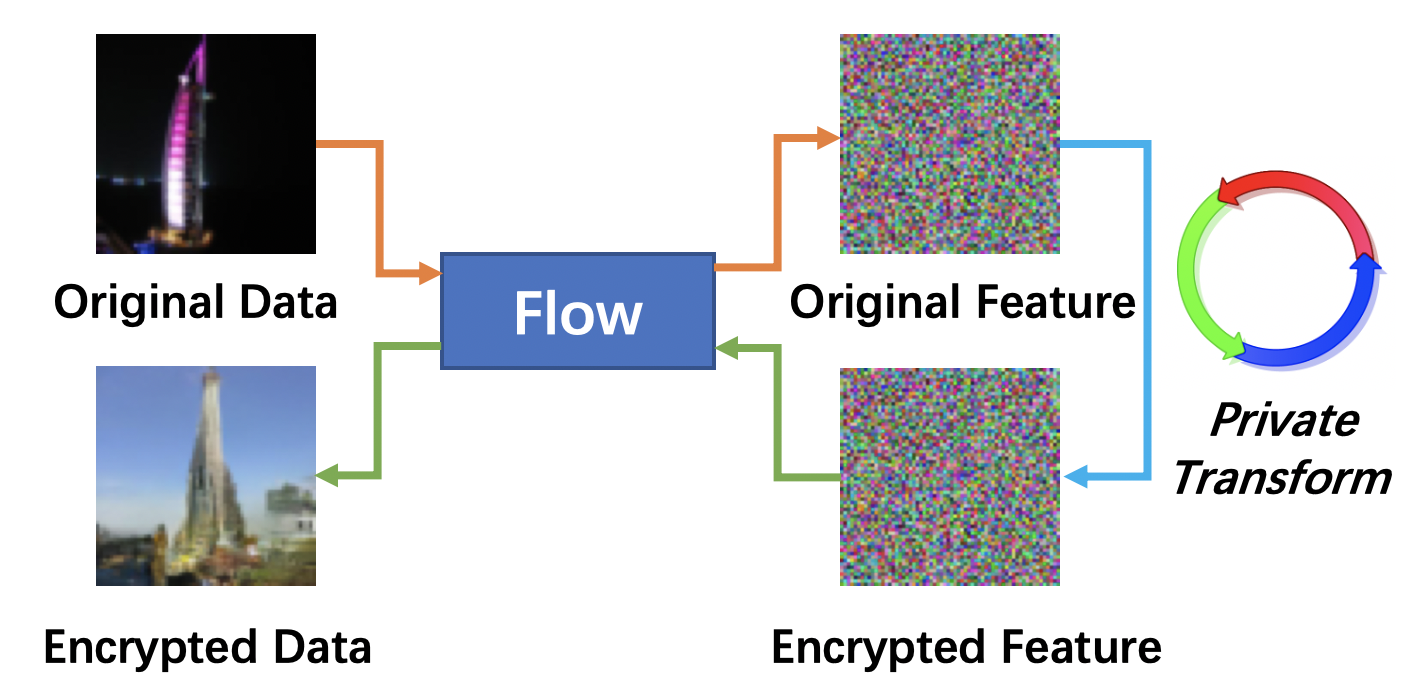}
    \caption{Illustration of our encrytion method. We first use a flow model to map the input data into the feature space, and then apply a private orthogonal transformation to the feature, and finally use the flow model to convert the encrypted feature back into the input space.}
    \label{fig:enc_method}
\end{figure}

We consider a variant of this problem, where instead of requiring the data to be \emph{completely private} so that no one gets any information about it, we only require data to be \emph{partially private}. That is, 
no one can efficiently recover the original data, 
but users can extract other useful information from the encrypted data. Although being different, 
our requirement has the flavor of differential privacy~\cite{dwork2006calibrating}, e.g., users can obtain the average salary of all employees, but cannot figure out the salary of each individual. 

Let us start with the following scenario: as a data provider, we want to hire a machine learning (ML) factory to train a model for us.
However, we do not trust the ML factory, and do not want it to sell the data or the trained model to the third party. Therefore, our encrypted data must be secure in the sense that no one can decrypt it. Moreover, the encrypted data can be used for machine learning training, but the trained model is only useful to the data providers with the secret encryption key, and is useless to the other parties. 
See Figure~\ref{fig:setting} for illustration.

To ensure security of the encryption method, 
we assume that the adversary is very powerful, i.e., it has unlimited computational power, and also knows exactly the underlying distribution of our input data (See Definition \ref{def:strong} for the formal definition).
Hence, an ideal encryption method in this setting should at least satisfy three properties:
\vspace{-0.1in}
\begin{itemize}\setlength{\itemsep}{0in}
    \item 
\textbf{Security:} 
Any adversary with the encrypted data, the distribution of the original input data, and unlimited computational power cannot decrypt our encrypted data.
\item 
\textbf{Integrity:} 
The encrypted data contain enough information for machine learning training purposes, compared with the original data. 
\item 
\textbf{Specificity:}
The models trained using the encrypted data 
only work for the encrypted test data, and  
cannot be directly applied to the original test data.
\end{itemize}


Using flow-based generative models~\cite{dinh2014nice,dinh2016density,kingma2018glow}, we propose an encryption method that satisfies these three properties with theoretical guarantees. Flow models are bijective functions that map the input data distribution into a Gaussian-like distribution, which we call the feature distribution. In our encryption method, we first privately sample a orthogonal matrix, and use it to apply an orthogonal transformation to the features of the input. After the transformation, we use the bijective flow model to convert the encrypted features back into the input space. 
See Figure~\ref{fig:enc_method} for illustration. 
The security of the encryption is ensured because the orthogonal matrix is private and cannot be recovered, as we will prove in Section~\ref{theory}.

Our encryption method is a general gadget that can be used 
in many different scenarios, not limited to the setting in Figure~\ref{fig:setting}.
For example, other parties may want to train a generative model based on our private data, 
but we do not want to share the data with them. As another example, when multiple parties are training a machine learning model together, 
they may face the data leakage problem when reporting the gradient during the training~\cite{zhu2019deep}. 
Our encryption method can be applied to both scenarios for making the private data secure. 
We will present more details about various applications of our encryption method in Section~\ref{sec:application}, as well as experimental results in Section~\ref{experiments}.
To our best knowledge, we are the first to encrypt data using invertible generative models.
We hope that our method will serve as a useful building block for data sharing and ML algorithm design in the future.

\section{Related Work}
\label{related}
Multi-party computation (MPC)~\cite{yao1982protocols,yao1986generate,goldreich2019play,chaum1988multiparty,ben2019completeness,bogetoft2009secure}
is a very powerful framework in the sense that it will keep all the input data private. 
Some of the results in MPC have been applied to the training process of neural networks: ~\citet{zhang2015privacy, li2017multi, aono2017privacy} used homomorphic encryption to encrypt the training for neural networks, and \citet{mohassel2017secureml, riazi2018chameleon}  developed new protocols for training neural networks in a 2PC setting. However, as we mentioned before, MPC-based protocols have huge computational overhead during the encryption/decryption process and people can only afford to train small networks on small datasets, so they are not yet practical in deep learning training.

There are also  papers that use differential privacy~\cite{dwork2014algorithmic} to preserve the data privacy during the training. These works only care about the privacy of the data, so the learned models can be directly applied to the original data distribution. \citet{abadi2016deep} added noise in the training to ensure differential privacy for the neural networks, \citet{vaidya2013differentially} used Naive Bayes to handle the case when there is only one data provider, and \citet{li2018differentially} extended that to the multiple-data-provider case.

Deep Generative models are useful tools of learning data distribution and generating new data. There are several kinds of deep generative models, e.g., Gerative Adversarial Networks~\cite{goodfellow2014generative}, Variational Auto-Encoders~\cite{kingma2013auto}, auto-regressive models~\cite{oord2016pixel}, and flow-based generators~\cite{kingma2018glow}. Among these, we use flow-based models because it is reversible. Flow-based generative models learn the data distribution mainly by warping Gaussian distribution and minimizing the log likelihood. To speed up the computation, people have taken several approaches, including using affine transformations~\cite{dinh2014nice,dinh2016density,kingma2018glow}, combining with auto-regressive methods~\cite{kingma2016improved}, and normalizing flows~\cite{rezende2015variational}.

Recently, \citet{huanginstahide} proposes InstaHide, which also aims to encrypt the private data while preserving the learnability to neural networks. They mixes private labeled images with public unlabeled images and applies a random sign-flip on the pixels to produce encrypted images which look like random noise. Our method, however, is invertible and hence guaranteed to preserve all the information of the original data.

\section{Preliminaries}
\label{theory-setting}
We use $\mathcal{D}$ to denote the underlying distribution of the input data points. 
Assume that our private data set contains $n$
data points $\{s_i\}_{i=1}^n\sim\mathcal{D}$, 
where $s_i \in \mathbb{R}^m$. 
We want to construct a mapping $Enc:\mathbb{R}^{m}\to\mathbb{R}^{m}$
for the data encryption. 
In this paper, we assume the data points are images. 
This is without loss of generality, because  our approach can easily generalize to other kinds of data.

\subsection{Total variation distance}
We use the following metric for measuring the distance between two distributions. 
\begin{definition}[Total variation distance]
For two distributions $P$ and $Q$ defined on domain $\mathcal{D}$, we define the total variation distance as 
\begin{equation}
\delta (P,Q)\triangleq \sup _{A\in {\mathcal {D}}}\left|P(A)-Q(A)\right|
\end{equation}
\end{definition}

For the total variation distance, we later need the following lemma that effectively removes the factor $n$ in the exponent.
\begin{lemma}[\cite{hoeffding1958distinguishability}, Eq.(4.5)]
\label{lem:sandwich}
For two distributions $P$ and $Q$, any integer $n\geq 1$, the following inequalities hold: 
\begin{equation}
\delta(P^n, Q^n)\leq 1-\left(1-\delta(P, Q)\right)^n\leq n\delta(P, Q).
\end{equation}
\end{lemma}


\subsection{Flow model}

Flow-based generative models~\cite{dinh2014nice,dinh2016density,kingma2018glow}, or simply flow models, are one kind of generative models. They assume that the data can be converted into a feature space by a bijective function $f$. That is, given any input $x\in \mathbb{R}^m$, there exists a corresponding latent variable $z\in \mathbb{R}^m$, such that $z=f(x)$, and $x=f^{-1}(z)$. Moreover, the latent variables follow some simple distributions, e.g., multivariate Gaussian distribution.

In other words, flow models can effectively convert the original data distribution (which does not have analytical form) into much simpler distributions like Gaussian distribution, and this mapping is invertible. There are many good flow models, and in this paper we mainly use the Glow model~\cite{kingma2018glow}. We remark that if better flow models are proposed in the future, they can also be applied in our setting to improve the empirical performance of the encryption. 


\subsection{Encryption Algorithm}
\label{alg}
Our encryption algorithm is  sketched in Algorithm~\ref{alg:enc}. Overall speaking, given the private data, we first train a flow model $f$, which learns a bijection between the input space and the feature space. After that, we map each original image to a feature, ``rotate'' that feature using a uniformly-sampled orthogonal matrix, and then ``recover'' the rotated feature into a new image, which is the encryption of the original image. 

\begin{algorithm}[htbp]
   \caption{Encryption of Private Image Data}
   \label{alg:enc}
\begin{algorithmic}
   \STATE {\bfseries Input:} Original images $\{s_i\}_{i=1}^n$, each of size $m$
   \STATE Train a flow model $f$ using $\{s_i\}_{i=1}^n$
   \STATE Sample $A$ uniformly from all $m\times m$ orthogonal matrices
   \FOR{$i=1$ {\bfseries to} $n$}
   \STATE $x_i=f(s_i)$
   \STATE $\hat{x}_i=Ax_i$
   \STATE $Enc(s_i)=f^{-1}(\hat{x}_i)$
   \ENDFOR
\end{algorithmic}
\end{algorithm}

Although being extremely simple, Algorithm~\ref{alg:enc} has many nice properties, as we will see in  Section~\ref{theory} and Section~\ref{sec:application}.
Based on the encryption algorithm, 
we can define various feature distributions. 

\begin{definition}[Feature distributions]
We define $\mathcal{G}_g$ to be the distribution of the features $x_i$, and $\mathcal{H}_0\triangleq\mathcal{N}(0,I_{m})$ is the standard normal distribution. Also, define $\mathcal{H}_1$ to be the distribution of rotated features $\hat{x}_i$.
\end{definition}

\subsection{Adversary and orthogonal matrix}
Below we formally define the power of our adversary. 
\begin{definition}[Strong Adversary]
\label{def:strong}
An adversary is called \textbf{strong} if it knows the encryption algorithm (except the private orthogonal matrix $A$), the encrypted data, the underlying distribution $\mathcal{D}$ of the original input data, and has unlimited computational power.
\end{definition}

This kind of adversary is very strong in the sense that it not only knows the encrypted data, but also have total information about the underlying distribution of the input. Besides, it can even use
the flow model $f$ in our encryption algorithm. The goal of this adversary is to recover the orthogonal matrix $A$ from these information. 
To formally define the notion of successful recovery, we first need to define volume and ball in the space of orthogonal matrices. 

\begin{definition}[Normalized volume]
Let $U$ be the set of all orthogonal matrices with size $m\times m$, and let $T$ be a subset of $U$. Define the normalized volume of $T$ be
\begin{equation}
v(T)\triangleq \mathbb{P}[A\in T],
\end{equation}
where $A$ is sampled uniformly from all $m\times m$ orthogonal matrices.
\end{definition}

\begin{definition}[$(\theta,A)$-ball]
Let $U$ be the set of all orthogonal matrices with size $m\times m$. 
Given an orthogonal matrix $A \in \mathbb{R}^{m\times m}$, $\delta>0$, the ball centered at $A$ with radius $\delta$ is defined as
\begin{equation}
\mathcal{B}_{\delta}(A)\triangleq \{
M | \| M-A\|_F \leq \delta , M\in U
\}.
\end{equation}
Given any $\theta\in (0,1]$,
let $\delta^*\triangleq 
\min \{\delta |
v(\mathcal{B}_{\delta}(A))\geq \theta\}$. 
We define $(\theta,A)$-ball as
$\mathcal{B}_{\delta^*}(A)$.
\end{definition}

Essentially, we use Frobenius norm to define a ball centered at matrix $A$ with normalized volume $\theta$. Here Frobenius norm is picked for convenience, and our analysis can be applied to other metrics as well. 
Now we can define whether a recovery is successful. 
\begin{definition}[successful recovery with tolerance $\theta$]
A recovery of matrix $A$ by an adversary is a \textbf{successful recovery with tolerance $\theta$} if the output $\tilde{A}$ of the adversary is within the $(\theta, A)$-ball. 
\end{definition}

In other words, if the output $\tilde{A}$ is close to $A$ in terms of Frobenius norm, we say the recovery is successful.  



\section{Main Theorem}
\label{theory}
In this section, we will first present the main theorem, and then give the proof for the simple case and the general case in Section~\ref{motivating-example} and Section~\ref{geneeral-case}, respectively. 



\begin{theorem}[Main Theorem]
\label{main-thm}
The success probability of any strong adversary trying to recover $A$ with tolerance $\theta$ cannot exceed $\delta(\mathcal{H}_0^n, \mathcal{H}_1^n)+\theta$.
\end{theorem}

Our main theorem gives an upper bound of the success probability of any strong adversary that tries to recover matrix $A$.
This upper bound has two terms, the $\theta$ term and the total variation term, which we elaborate below.  

The $\theta$ term is completely controlled by the data provider.  For example, if $\theta=0.5$, that means we 
define the recovery to be successful if the adversary can find an orthogonal matrix that is closer to $A$ compared with at least half of all possible orthogonal matrices. 
In this case, the adversary can simply do a random guess to get $0.5$ success probability. Indeed, the \emph{$\theta$ term is inevitable}, because adversary can always randomly sample matrix that is in $(\theta,A)$-ball with probability $\theta$. 

The total variation term measures the distance between 
the feature distribution and the Gaussian distribution.
Intuitively, 
the Gaussian distribution is 
the most difficult case for the adversary because it is rotation-invariant. In other words, before and after applying the orthogonal transformation, the data look the same to the adversary, which makes the recovery of the exact $A$ impossible. 
Therefore we have this term in the upper bound, indicating that if the distance to the Gaussian distribution is smaller, it will be harder to decrypt the matrix $A$.
In practice, a better flow model will make this distance smaller.

Using Lemma~\ref{lem:sandwich}, we immediately get the following corollary, which removes the exponent $n$ in the upper bound. 
\begin{corollary}
The success probability of any adversary trying to recover $A$ with tolerance $\theta$ cannot exceed $n\delta(\mathcal{H}_0, \mathcal{H}_1)+\theta$.
\end{corollary}

\subsection{Proof of the simple case}
\label{motivating-example}

Let's start the proof with a simple example where $\delta(\mathcal{H}_0, \mathcal{H}_1)=0$, i.e., the distribution $\mathcal{G}_g$ is exactly equal to standard normal distribution. In this case, $\delta(\mathcal{H}_0^n, \mathcal{H}_1^n)=0$, so our main theorem can be simplified to the following lemma.

\begin{lemma}
\label{toy-lem}
If $\mathcal{G}_g=\mathcal{H}_0\triangleq\mathcal{N}(0,I_{m})$, then the success probability of any adversary trying to recover $A$ with tolerance $\theta$ cannot exceed $\theta$.
\end{lemma}
\begin{proof}
If $\mathcal{G}_g=\mathcal{N}(0,I_{m})$, i.e., $\forall i\in[n], x_i\sim\mathcal{N}(0,I_{m})$, we can compute the distribution of $\hat{x_i}$:
\begin{align}
\mathbb{E}[\hat{x_i}]&=A\mathbb{E}[x_i]=A\cdot 0 = 0.\\
Var(\hat{x_i}) &= A\cdot Var(x_i)A^\top = AA^\top = I_{m}.
\end{align}
Note that the distribution of $\hat{x_i}$ is still a Gaussian distribution, and it has the same expectation and variance as $x_i$. Thus, 
\begin{equation}
\mathcal{H}_0=\mathcal{H}_1.
\end{equation}

Therefore, for all orthogonal matrix $A$, the distribution of $Enc(s_i)$ are the same, i.e., $A$ is independent of the distribution of the encrypted images. This means that given the encrypted images, any output from the adversary is equivalent to a random guess over all orthogonal matrices. From the definition of tolerance, we know that the success probability of any adversary trying to recover $A$ with tolerance $\theta$ cannot exceed $\theta$.
\end{proof}

This lemma tells us that if the distribution of the features is exactly standard normal distribution, i.e., the first term in the upper bound is 0, then the only possibility that the adversary can succeed is that the random guess is within the tolerance. In the next part, we will provide a formal proof for the general case, i.e., our main theorem.

\subsection{Proof of the general case}
\label{geneeral-case}

In this part, we will show the proof of our main theorem, which is a generalization of the proof for Lemma \ref{toy-lem}.

We first provide the intuition of this proof: What the adversary can observe is only the encrypted data, which can be considered as some samples of the distribution of $Enc(s_i)$. From the proof of Lemma \ref{toy-lem} we know that if $\mathcal{G}_g=\mathcal{N}(0,I_{m})$, the adversary will output a random guess of matrix $A$. Thus, given $n$ samples of the encrypted data distribution, if the adversary cannot \textbf{distinguish} between $\mathcal{H}_1$ and $\mathcal{H}_0$, it will be impossible for him to do better than random guess in recovering matrix $A$.

Let us formalize this intuition into a full proof:
\begin{proof}[Proof of Theorem \ref{main-thm}]

Firstly, we restrict this problem in feature space: The data-to-feature map (flow model) is a bijection and known to the adversary, and the adversary has unlimited computational power. Thus, given any data or distribution in the image space, the adversary can map it to the feature space using the data-to-feature map, and vice versa. Therefore, it suffices to only consider the data and distributions in the feature space.

We will prove that if we can recover the matrix $A$ with high probability, then we can distinguish very similar distributions with high probability, which is impossible. Below we first formally define the two related problems, and then we will do a recursion between them. 

\textbf{Problem 1 (distinguishing distributions):} Given the Gaussian distribution $\mathcal{H}_0=\mathcal{N}(0,I_{m})$,  feature distribution $\mathcal{G}_g$, rotated feature distribution $\mathcal{H}_1$, and matrix $A$ uniformly sampled from the set of all orthogonal matrices. Sample a bit $b$ uniformly from $\{0, 1\}$ (this bit is hidden from the solver), and take $n$ i.i.d. samples $\{x_i\}_{i=1}^n\sim\mathcal{H}_b^n$. We want to find $b$.

\textbf{Problem 2 (recovering matrix $A$):}
Given a feature distribution $\mathcal{G}_g$, and samples of rotated features $\{\hat{x_i}\}_{i=1}^n\sim (\mathcal{H}_1)^n$ and the  we want to recover the rotation matrix $A$.

For a strong oracle $O_2$ for \textbf{Problem 2}, assume that its success probability is $p$. 
Note that this success probability is the expectation taken over the samples $\{\hat{x_i}\}_{i=1}^n$, the choice of orthogonal matrix $A$, and the randomness of $O_2$ itself. 

We remark that $O_2$ can also take $\mathcal{G}_g$ and 
$\{\hat {x_i}\}_{i=1}^n\sim (\mathcal{H}_0)^n$ as the input because the input shape matches. However, since \textbf{Problem 2} expects
$\hat {x_i}$ to be sampled from $H_1$, if we use $H_0$ instead, $O_2$ may generate unexpected outputs without any guarantees on the success probability. 

We now use this strong oracle $O_2$ to construct an oracle $O_1$ for \textbf{Problem 1}: Let
\begin{equation}
    O_1=\mathbbm{1}_{\text{success}}[O_2(G_g, x_1, \cdots, x_n)],
\end{equation}

In other words, $O_1$ first assumes $b=1$, and then 
calls $O_2$ with the feature distribution and the $n$ samples, receive the output matrix $\tilde A$ of $O_2$.
$O_1$ outputs $1$ if 
$\tilde A$ is a successful recovery with tolerance $\theta$ to $A$, and $0$ otherwise.

Define the success probability of $O_1$ be the expected probability of returning the correct bit $b$, where the expectation is taken over the randomness of bit $b$ and the oracle $O_2$.

When $b=0$, i.e., the samples come from the Gaussian distribution, from the proof of Lemma \ref{toy-lem} we know that
no algorithm can get any information about $A$, i.e., the output of $O_2$ is independent with $A$. Since $A$ is a randomly sampled orthogonal matrix, 
we know that when $b=0$, 
the success probability of  $O_1$(i.e., the probability of outputting $0$) is $1-\theta$.


When $b=1$, we know that its success probability is equal to the success probability of $O_2$, which is $p$. Therefore, the overall success probability of $O_1$ is  $\frac{(1-\theta)+p}{2}$, as $b$ is uniformly sampled from $\{0,1\}$.

However, from the Bayes risk of binary hypothesis test~\cite{wainwright2019high}, we know that the maximum success probability for \textbf{Problem 1} is $\frac{1+\delta(\mathcal{H}_0^n, \mathcal{H}_1^n)}{2}$.

Thus,
\begin{equation}
\frac{(1-\theta)+p}{2}\leq \frac{1+\delta(\mathcal{H}_0^n, \mathcal{H}_1^n)}{2},
\end{equation}
i.e.,
\begin{equation}
    p\leq \delta(\mathcal{H}_0^n, \mathcal{H}_1^n)+\theta.
\end{equation}
\end{proof}

\section{Applications}
\label{sec:application}
In this section we discuss various applications of our encryption method. We believe that this method can serve as an important building block for
protecting privacy in machine learning. 

\subsection{Supervised secure model training}
Our first example is the setting described in Figure~\ref{fig:setting}. To solve this problem, we may directly apply our Algorithm~\ref{alg:enc} to each class of images. That is, assuming that there are $N$ classes,  we train a flow model $f_i$ for the $i$-th class.
In the dataset, for each pair of data points $(x,y)$, we encrypt $x$ to get $\hat x$ using $f_y$, and send $(\hat x, y)$ to the ML factory. Notice that we cannot train a simple flow model for all classes because after the whole dataset encryption using a single model, the input-label correspondence might be lost, but independent encryption for each class does not have this problem.  

This solution  satisfies the three properties that we required: Security, Integrity and Specificity. The Security property holds because the ML factory cannot decrypt the orthogonal matrix for each class, as we shown in Theorem~\ref{main-thm}. Therefore, it cannot recover the original data points. 

The Integrity property 
holds because we apply orthogonal transformation in the feature space, which has semantic meanings. After the transformation, the encrypted features follow a similar Gaussian-like distribution, so they still contain enough information for machine learning training purposes. We will further empirically demonstrate this property in our experiments.

We verify the Specificity property of our method by running simulations, 
as we will see in Section~\ref{experiments:supervised}. That is, after being trained on the encrypted data, the model does not have good prediction accuracy on the original data distribution. In order to apply the model, one needs to first encrypt the data using the secret orthogonal matrix $A$. In other words, only the data provider can use the trained model. 




\subsection{Unsupervised generative model training}
\label{sec:app:unsup}
Consider the scenario that the ML factory wants to learn a generative model using the private unlabeled data from the data provider. The data provider agrees to let ML factory train the model, but does not want it to see the private dataset. In this case, we can apply our Algorithm~\ref{alg:enc} to encrypt the data, and ask the ML factory to train the generative model using the encrypted data. The intuition is, after the orthogonal transformation in the feature space, the encrypted data still contain enough semantic information for the data distribution (the Integrity property). 

As we demonstrate in Section~\ref{experiments:unsupervised}, this solution works well, and the encrypted data can indeed be used for training a very good generative model.

\subsection{Data leakage with gradients}
\label{sec:app:leak}
In this scenario, multiple agents collaborate together to train a model, and each agent has its own private data. As argued by \citet{zhu2019deep}, each agent faces the problem of data leakage when sharing the gradients with the other agents.


Specifically, for the synchronous distributed training, 
in $t$-th iteration , agent $i$ samples a minibatch $(x_{t,i},y_{t,i})$ to compute the gradients $\nabla W_{t,i}$
\begin{equation}
    \nabla W_{t,i} = \frac{\partial \ell(F(x_{t,i},W_{t}),y_{t,i})}{\partial W_{t}}
\end{equation}
where $\ell$ is the loss function, $F$ is the neural network function, and $W_t$ is the synchronized weight for the network at the iteration $t$. Then the gradients are averaged across the $N$ agents, so the weight $W_t$ is updated as:
\begin{equation}
    W_{t+1} = W_{t} - \eta \frac{1}{N}\sum_{j=1}^N \nabla W_{t,j}
\end{equation}

The data leakage happens if an adversary has the information of
the gradients, model architecture and weights. In this case, the input data can be recovered by minimizing the distance between the actual gradients and gradients obtained from the randomly initialized dummy input $x'$ and label $y'$~\cite{zhu2019deep}, 
where the optimization is applied to $x'$ and $y'$. The gradient from $x'$ and $y'$ is defined as
\begin{equation}
    \nabla W' = \frac{\partial \ell(F(x',W),y')}{\partial W}
\end{equation}

Our encryption can help avoid data leakage under the assumption that  all the agents can be trusted, and they share the common private orthogonal matrices for each class of data. In this setting, all the data points can be encrypted, and the adversary will not be able to learn the private data from the gradients without knowing the orthogonal matrices. As we will show in Section~\ref{experiments:leakage}, the adversary can only recovery the encrypted data using the gradients. 


\begin{figure}[thbp]
\centering

\subfigure[Original Image]{
\begin{minipage}[t]{0.3\linewidth}
\centering
\includegraphics[width=0.7in]{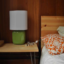}
\end{minipage}%
\begin{minipage}[t]{0.3\linewidth}
\centering
\includegraphics[width=0.7in]{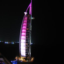}
\end{minipage}%
\begin{minipage}[t]{0.3\linewidth}
\centering
\includegraphics[width=0.7in]{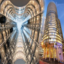}
\end{minipage}
}%

\subfigure[Encrypted Image]{
\begin{minipage}[t]{0.3\linewidth}
\centering
\includegraphics[width=0.7in]{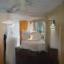}
\end{minipage}
\begin{minipage}[t]{0.3\linewidth}
\centering
\includegraphics[width=0.7in]{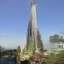}
\end{minipage}
\begin{minipage}[t]{0.3\linewidth}
\centering
\includegraphics[width=0.7in]{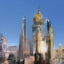}
\end{minipage}
}%

\centering
\caption{ \label{fig:encyrpte-result}\textbf{Visualization of the images before and after encryption}. The first row contains original images and the second row contains the corresponding encrypted images. 
}
\end{figure}

\section{Experiments}
\label{experiments}
In this section, we present the experimental results of our method for different scenarios discussed in Section~\ref{sec:application}: supervised secure model training (Section~\ref{experiments:supervised}), unsupervised generative model training (Section~\ref{experiments:unsupervised}) and data leakage with gradients (Section~\ref{experiments:leakage}). 
We use PyTorch~\cite{paszke2019pytorch} for all the experiments.


\textbf{Glow Model~\cite{kingma2018glow}.}  
We use Glow model in our experiments. 
Glow model can learn distribution on unbounded space, but the data have bounded interval range. To reduce the impact of boundary effects, we model the density as logit(\(\alpha + (1-\alpha)\bigodot{x/256}\)) following~\cite{dinh2016density}. 
In our implementation, 
we use activation normalization as our normalization method, which is a scale and bias layer with data dependent initialization. For permutation, we use an (learned) invertible 1 $\times$ 1 convolution, where the weight matrix is initialized as a random rotation matrix. Also, we use the Affine coupling layer instead of the Additive coupling layer. We set \#levels to $3$, steps to $16$, learning rate to $0.001$ and batch size to $512$.

We visualize some figures before and after the encryption in Figure~\ref{fig:encyrpte-result}. 
As we can see, original images and their corresponding encrypted images visually belong to the same class. This shows that the Glow model can indeed capture the semantic meanings of the inputs.

\subsection{Supervised secure model training}
\label{experiments:supervised}
In this subsection, we run  experiments on CIFAR10 and LSUN with Glow and classifiers including ResNet~\cite{he2016deep} and DenseNet~\cite{huang2017densely}.  

\begin{table}[th]
\begin{center}
\caption{\label{tab:Sup-learn} \textbf{Test accuracy of different models on different datasets.} There shows the test accuracy of different models on different datasets to validate the effectiveness of our method on supervised learning.}

\begin{tabular}{ |c|c|c|c| } 
\hline
\multicolumn{2}{|c|}{}&\multicolumn{2}{|c|}{Neural Network}\\
\hline
TrainSet& TestSet& ResNet50 & DenseNet121 \\
\hline
\multirow{2}{4em}{Original CIFAR}&Ori-CIFAR & 95.51 & 95.54 \\ 
&Enc-CIFAR & 44.49 & 45.83 \\ 
\hline
\multirow{2}{4em}{Encrypted CIFAR}&Ori-CIFAR & 33.69 & 38.45 \\ 
&Enc-CIFAR & 88.53 & 88.73 \\ 
\hline
\multirow{2}{4em}{Original LSUN}&Ori-LSUN & 84.86 & 85.99 \\ 
&Enc-LSUN & 69.00 & 69.77 \\ 
\hline
\multirow{2}{4em}{Encrypted LSUN(1)}&Ori-LSUN & 55.83 & 58.33 \\ 
&Enc-LSUN & 92.31 & 93.21 \\ 
\hline
\multirow{2}{4em}{Encrypted LSUN(2)}&Ori-LSUN & 25.53 & 30.10 \\ 
&Enc-LSUN & 77.96 & 78.69 \\ 
\hline
\end{tabular}
\end{center}
\end{table}

\textbf{Experiment Details.}  We train our model on two natural image datasets: CIFAR-10~\cite{krizhevsky2009learning} and LSUN~\cite{yu15lsun}. For every dataset, we train a separate Glow model for each individual class. Every CIFAR-10 class has $5,000$ data points, while LSUN is much larger. Except $126,227$ data points for the class of outdoor church and $168,103$ data points for the class of classroom, we use randomly sampled $150,000$ data points for all other classes to train Glow models. We call this case  LSUN(1).
We also tried another experiment where each Glow model uses only $5,000$ data points for training, which is called  LSUN(2). Essentially, in the first setting, the encryption method uses a more powerful Glow model trained with more data points.



All the images are resized to $32*32$ and all Glow models are using the same hyper parameters as mentioned above. 
For LSUN, we sampled $60,000$ data points for model training ($50,000$) and testing ($10,000$).
As we proved in Theorem~\ref{main-thm}, our encryption method is secure. Therefore, it remains to see whether it also has the Integrity and Specificity property, which we elaborate below.

\textbf{Integrity.}
In Table~\ref{tab:Sup-learn}, we present the experimental results of different models on different datasets. As we can see, after being trained on the encrypted data, both ResNet and DenseNet can achieve very good accuracy on the encrypted test data, i.e., $88\%$ for CIFAR-10, and $92\%-93\%$ for LSUN(1). It is worth pointing out that the test accuracy for the encrypted case is not as high as the accuracy for the original case for CIFAR-10. We believe that this is because CIFAR-10 is a small dataset, and cannot be used for training good Glow models. By contrast, in LSUN(1), when we use very large dataset to train the Glow models, the accuracy is increased significantly, even more than the accuracy in the original case. This is because the original case does not have the extra information from Glow models, and only relies on the information from the $50,000$ training data. As a result, in LSUN(2), when the Glow models are trained using $5,000$ data points each, the accuracy in the encrypted case dramatically decreased. Hence, when well trained flow models are available, the Integrity property can be ensured in practice. 

\textbf{Specificity.}
In Table~\ref{tab:Sup-learn}, when the models are trained using encrypted datasets, they will have good accuracy on encrypted test data, but bad  accuracy on the original test data. In other words, the trained model can only be used by the data provider who knows how to encrypt data, and cannot be used by other parties.

\begin{table}[th]
\centering
\caption{\label{tab:compare-glow}\textbf{Comparing the Glows based on BPDs.}  In the table below, we present the BPDs of CIFAR-10 and ImageNet calculated by the two glows where the first glow is trained on the original CIFAR and the second glow is trained on the encrypted CIFAR. }
\begin{tabular}{ccc}
\hline
Glows-Similarity& CIFAR-10 & ImageNet\\
\hline
Glow-original& 3.49& 3.76\\
Glow-encrypted& 3.67& 3.93 \\
Glow-untrained& 9.75& 9.66 \\
\hline
\end{tabular}
\end{table}

\subsection{Unsupervised generative model training}
\label{experiments:unsupervised}

In this subsection, we present experimental results for the unsupervised generative model training described in Section~\ref{sec:app:unsup}. We hope that the generative model trained on the encrypted data is as good as the one trained on the original data. 
For simplicity, we also use Glow as the generative model for training. But in practice, one can also use the encrypted data to train GANs or VAEs. 
Below we show 
both quantitative  and qualitative experimental evaluations.

\begin{figure}[htbp]
\centering

\subfigure[Original Image]{
\begin{minipage}[t]{0.17\linewidth}
\centering
\includegraphics[width=0.4in]{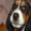}
\end{minipage}%
\begin{minipage}[t]{0.17\linewidth}
\centering
\includegraphics[width=0.4in]{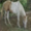}
\end{minipage}%
\begin{minipage}[t]{0.17\linewidth}
\centering
\includegraphics[width=0.4in]{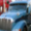}
\end{minipage}
\begin{minipage}[t]{0.17\linewidth}
\centering
\includegraphics[width=0.4in]{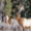}
\end{minipage}
\begin{minipage}[t]{0.17\linewidth}
\centering
\includegraphics[width=0.4in]{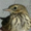}
\end{minipage}
}%

\subfigure[Images Reversed by The Second Glow]{
\begin{minipage}[t]{0.17\linewidth}
\centering
\includegraphics[width=0.4in]{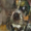}
\end{minipage}
\begin{minipage}[t]{0.17\linewidth}
\centering
\includegraphics[width=0.4in]{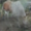}
\end{minipage}
\begin{minipage}[t]{0.17\linewidth}
\centering
\includegraphics[width=0.4in]{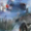}
\end{minipage}
\begin{minipage}[t]{0.17\linewidth}
\centering
\includegraphics[width=0.4in]{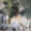}
\end{minipage}
\begin{minipage}[t]{0.17\linewidth}
\centering
\includegraphics[width=0.4in]{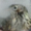}
\end{minipage}
}%

\centering
\caption{ \label{fig:compare-glows}\textbf{Comparing the two Glows trained on original data and encrypted data.} The first row are original CIFAR-10 images. We extract their features by the glow trained on original CIFAR-10 images, and reverse the features by the glow trained on encrypted CIFAR-10 images to get the second row's images. }
\end{figure}

\textbf{Quantitative Evaluations.}  For quantitative evaluations, we use 
Bit per dimension (BPD) to evaluate the similarity. BPD is defined as the negative log likelihood (with base 2) divided by the number of pixels.

In this experiment, we train Glow models in both original dataset and  encrypted dataset. Notice that the models are trained using data points from all classes. The results are shown in Table~\ref{tab:compare-glow}. As we can see, the BPDs of both original dataset and enrypted dataset are close, for both CIFAR-10 and ImageNet. By contrast, the BPDs for Glows without any training (i.e., at its random initialization) are much larger. In other words, from the BPD perspective, the Glow models trained on original and encrypted datasets are similar.

\textbf{Qualitative Experiments.} 
In this experiment, we train two Glow models, on the original CIFAR-10 and the encrypted CIFAR-10 datasets, respectively. We then first extract the feature of  a given original image using the first model, and then use the second model to recover the original image. See Figure~\ref{fig:compare-glows} for the final results. 

Although they are not exactly the same (because the two Glow models are different), we can tell that they are visually similar to each other. We believe that this shows the two models contain lots of common information. Since our encryption method ensures the security of the training data, it is a good way to let others to train  generative models without sharing the training data. 



\subsection{Data leakage with gradients}
\label{experiments:leakage}

\begin{table}[ht]
\centering
\caption{\label{tab:leak-rotate-CIFAR-10}\textbf{Classification Results on Leaked encrypted CIFAR-10.}
We steal the encrypted CIFAR-10 images from gradients and use the leaked images to train classifiers. This table shows the test accuracies.}
\begin{tabular}{ccc}
\hline
CIFAR-10& Original Test& Encrypted Test\\
\hline
ResNet50& 9.49& 42.97\\
DenseNet121& 10.3& 39.11 \\
\hline
\end{tabular}
\end{table}

\begin{table}[ht]
\centering
\caption{\label{tab:leak-rotate-lsun}\textbf{Classification Results on Leaked encrypted LSUN(1).}  We steal the encrypted LSUN(1) images from gradients and use the leaked images to train classifiers. This table shows the accuracies. }
\begin{tabular}{ccc}
\hline
LSUN-Sampled& Original Test& Encrypted Test\\
\hline
ResNet50& 46.39& 59.40 \\
DenseNet121& 44.79& 51.35 \\
\hline
\end{tabular}
\end{table}

\begin{figure}[htbp]
\centering

\subfigure[Original Image]{
\begin{minipage}[t]{0.3\linewidth}
\centering
\includegraphics[width=0.5in]{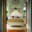}
\end{minipage}%
}%
\subfigure[Encrypted Image]{
\begin{minipage}[t]{0.35\linewidth}
\centering
\includegraphics[width=0.5in]{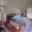}
\end{minipage}%
}%
\subfigure[Leaked Image]{
\begin{minipage}[t]{0.3\linewidth}
\centering
\includegraphics[width=0.5in]{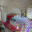}
\end{minipage}%
}%

\centering
\caption{ \label{fig:leak-figure}Visualization result of original data, encrypted data and leaked data.}
\end{figure}

In this subsection, we run experiments for data leakage with gradients, as discussed in Section~\ref{sec:app:leak}.

\textbf{Experimental Setup.}  
Our experimental settings are generally following \citet{zhu2019deep}.
We use L-BFGS\cite{liu1989limited} as our optimizer with learning rate $1$. We set batch size $1$ and iteration $100$ for every image. To meet the requirements that the model needs to be twice-differentiable, we replace activation ReLU to Sigmoid and remove strides. For the labels, we randomly initialize a continuous vector instead of directly optimizing the discrete vector. For the inputs, we randomly initialize a continuous vector with Shape $N\times C\times H \times W$. Here the $C=3$, $N=1$ and $H=W=32$.  
For convenience, all the experiments are using randomly initialized weights. We apply the method described in Section~\ref{sec:app:leak}, i.e., all the agents will share the same secret orthogonal matrices, and use those matrices for data encryption.

\textbf{Experimental Results.}  In Figure \ref{fig:leak-figure}, We show 
the leaked images using gradient information. 
The recovery of the encrypted image is successful, but as we proved in Theorem~\ref{main-thm}, the adversary cannot further recovery the original image from the encrypted image. 


We also steal all the encrypted images using the gradient information for both CIFAR-10 and LSUN(1) datasets, and use those  images for training. See Table~\ref{tab:leak-rotate-CIFAR-10} and Table \ref{tab:leak-rotate-lsun} for details. 
We can see that the trained model has very bad accuracies on the original test data, and slightly higher test accuracy on the encrypted test data. 

To sum up, 
leaking data from gradients can only get the encrypted data, so our encryption keeps the private data safe. Moreover, from our experimental results, the accuracy of model trained on leaked data does not work on the original test data, therefore is useless to the adversary for classification tasks.

\section{Conclusion}

In this paper, we proposed a data encryption algorithm for theoretically ensuring the privacy while preserving useful information from the data. 
Our method is a building block which can be applied in many scenarios in machine learning. We give theoretically guarantees for the security of our encrypted algorithm, and our experiments demonstrated that our methods are useful in real datasets like LSUN and CIFAR-10. For the future work, we hope to apply our method to other interesting problems. 
\section*{Acknowledgements}
This work has been supported in part by the Zhongguancun Haihua Institute for Frontier Information Technology.
Also thank Jiaye Teng, Haowei He and Zixin Wen for the helpful discussions.



\bibliography{bibfile}
\bibliographystyle{icml2020}


\end{document}